\documentclass{article} % For LaTeX2e
\usepackage{iclr2020_conference,times}

\usepackage{amsmath,amsfonts,bm}

\def\eqref#1{equation~\ref{#1}}
\def\1{\bm{1}}

\DeclareMathAlphabet{\mathsfit}{\encodingdefault}{\sfdefault}{m}{sl}
\SetMathAlphabet{\mathsfit}{bold}{\encodingdefault}{\sfdefault}{bx}{n}

\usepackage{hyperref}
\usepackage{url}
\usepackage{graphicx}
\usepackage{caption}
\usepackage{subcaption}
\usepackage{wrapfig, booktabs}
\usepackage{amsthm}
\newtheorem{theorem}{Theorem}
\newtheorem{definition}{Definition}
\newtheorem{remark}{Remark}
\title{Intrinsic Dimensionality Explains the Effectiveness of Language Model Fine-Tuning}
\author{Armen Aghajanyan, Luke Zettlemoyer, Sonal Gupta \\
Facebook\\
\texttt{\{armenag,lsz,sonalgupta\}@fb.com} \\
}

\iclrfinalcopy % Uncomment for camera-ready version, but NOT for submission.
\begin{document}

\maketitle

\begin{abstract}
Although pretrained language models can be fine-tuned to produce state-of-the-art results for a very wide range of language understanding tasks, the dynamics of this process are not well understood, especially in the low data regime. Why can we use relatively vanilla gradient descent algorithms (e.g., without strong regularization) to tune a model with hundreds of millions of parameters on datasets with only hundreds or thousands of labeled examples? 
In this paper, we argue that analyzing fine-tuning through the lens of intrinsic dimension provides us with empirical and theoretical intuitions to explain this remarkable phenomenon. We empirically show that common pre-trained models have a very low intrinsic dimension; 
in other words, there exists a low dimension reparameterization that is as effective for fine-tuning as the full parameter space.  
For example, by optimizing only 200 trainable parameters randomly projected back into the full space, we can tune a RoBERTa model to achieve 90\% of the full parameter performance levels on MRPC. Furthermore, we empirically show that pre-training implicitly minimizes intrinsic dimension and, perhaps surprisingly, larger models tend to have lower intrinsic dimension after a fixed number of pre-training updates, at least in part explaining their extreme effectiveness. 
Lastly, we connect intrinsic dimensionality with low dimensional task representations and compression based generalization bounds to provide intrinsic-dimension-based generalization bounds that are independent of the full parameter count.
\end{abstract}

\section{Introduction}
Pre-trained language models~\citep{GPT, BERT, ROBERTA, BART,MARGE}  provide the defacto initialization for modeling most existing NLP tasks. However, the process of fine-tuning them on often very small target task datasets remains somewhat mysterious. Why can we use relatively vanilla gradient descent algorithms (e.g., without strong regularization) to tune a model with hundreds of millions of parameters on datasets with only hundreds or thousands of labeled examples?  

We propose intrinsic dimensionality as a new lens through which fine-tuning can be analyzed \citep{intrinsic_dimension}. An objective function's intrinsic dimensionality describes the minimum dimension needed to solve the optimization problem it defines to some precision level. In the context of pre-trained language models, measuring intrinsic dimensional will tell us how many free parameters are required to closely approximate the optimization problem that is solved while fine-tuning for each end task. For example, we will show that 200 parameters (randomly projected back into the full parameter space) are enough to represent the problem of tuning a RoBERTa model to within 90\% of the performance of the full model. More generally, we also describe a set of strong empirical and theoretical connections between intrinsic dimensionality, number of parameters, pre-training, and generalization.

We first empirically show that standard pre-trained models can learn a large set of NLP tasks with very few parameters and that the process of pre-training itself implicitly minimizes the intrinsic dimension of later tuning for different NLP tasks. We continue by conducting a study across over a dozen various pre-trained models to show that number of parameters strongly inversely correlates with intrinsic dimensionality, at least in part to justify the extreme effectiveness of such models. We interpret pre-training as providing a framework that learns how to compress the average NLP task.  Finally, we connect intrinsic dimensional with low dimensional task representations and compression based generalization bounds to provide intrinsic-dimension-based generalization bounds that are independent of the full parameter count, further justifying why these methods generalize so well in practice across tasks.

The contributions of our paper are the following:
\begin{itemize}
    % \item We discuss and augment an existing method for calculating the intrinsic dimension for fine-tuning objectives within the context of Transformer models. Specifically we propose a new way to compute a Structure Aware Intrinsic Dimension (\textbf{SAID}).
    \item We empirically show that common NLP tasks within the context of pre-trained representations have an intrinsic dimension several orders of magnitudes less than the full parameterization.
    \item We propose a new interpretation of intrinsic dimension as the downstream fine-tuning task's minimal description length within the framework of the pre-trained model. Within this interpretation, we empirically show that the process of pre-training implicitly optimizes the description length over the average of NLP tasks, without having direct access to those same tasks.
    \item We measure the intrinsic dimension of a large set of recently developed pre-training methods. We discover that there exists a  fortuitous trend where larger models tend to have a smaller intrinsic dimension.
    \item Lastly, we show that compression based generalization bounds can be applied to our intrinsic dimension framework to provide generalization bounds for large pre-trained models independent of the pre-trained model parameter count.
\end{itemize}

\section{Related Work}
Calculating the intrinsic dimension of an objective function was proposed \cite{intrinsic_dimension}. In their paper, they analyzed the impact of various architectures on the intrinsic dimensionality of their objective. Our work is a direct extension of this paper, focusing on analyzing pre-trained representations instead.

There is a large collection of literature analyzing pre-trained models from the perspective of capacity. For example, a recent line of work has shown that pre-trained models such as BERT are redundant in their capacity, allowing for significant sparsification without much degradation in end metrics \citep{bert_lottery_ticket, bert_lottery_all_winners, hongyuan_lotter_ticket}. \cite{adapter_network} showed that fine-tuning top layers of pre-trained models is not effective and that alternate methods allow fine-tuning effectively with a couple of percent of the parameters. Furthermore, we can view computing the intrinsic dimensionality as a continuous relaxation of the sparsification problem.

Moreover, standard approaches towards fine-tuning seem to have non-trivial effects on the generalization of pre-trained representations \citep{RXF}. A holistic explanatory picture of the successes of fine-tuning has not yet been painted. A clear understanding of the underlying mechanisms which lead to the incredible generalization of fine-tuned pre-trained representations is currently missing. Moreover, we still do not understand why various pre-training methodology manifests in universally useful representations.

\section{Intrinsic Dimensionality of Finetuning}

\paragraph{Background}
An objective function's intrinsic dimension measures the minimum number of parameters needed to reach satisfactory solutions to the respective objective \citep{intrinsic_dimension}. Alternatively, the intrinsic dimension represents the lowest dimensional subspace in which one can optimize the original objective function to within a certain level of approximation error. Computing the exact intrinsic dimensional of the objective function is computation intractable; therefore, we resort to heuristic methods to calculate an upper bound.
Let $\theta^{D}=\left[\theta_0, \theta_1,..., \theta_m\right]$ be a set of $D$ parameters that parameterize some model $f(\cdot, \theta)$. Instead of optimizing the empirical loss in the original parameterization ($\theta^{D}$), the subspace method fine-tunes the model via the following re-parametrization in the lower-dimensionsal $d$-dimensions:
\begin{equation}
    \theta^{D} = \theta^{D}_0 + P(\theta^{d})
    \label{eq:subspace_def}
\end{equation}
where $P: \mathbb{R}^d \rightarrow \mathbb{R}^D$ projects from a parameter from a lower dimensional $d$ to the higher dimensional $D$. Intuitively, we do an arbitrary random projection onto a much smaller space; usually, a linear projection, we then solve the optimization problem in that smaller subspace. If we reach a satisfactory solution, we say the dimensionality of that subspace is the intrinsic dimension. This methodology was proposed in the seminal paper by \cite{intrinsic_dimension}. Concretely \cite{intrinsic_dimension} proposed 3 various actualizations of $P$; a random linear dense projection ($\theta^{d}W$), random linear sparse projection($\theta^{d}W_{\text{sparse}}$) and random linear projection via the Fastfood transform \citep{fastfood}.

We will primarily use the Fastfood transform, defined as:
\begin{align}
    \theta^{D} = \theta^{D}_0 + \theta^{d}M && M=HG\Pi HB \label{eq:did}
\end{align}
The factorization of $M$ consists of $H$, a Hadamard matrix, $G$, a random diagonal matrix with independent standard normal entries, $B$ a random diagonal matrix with equal probability $\pm 1$ entries, and $\Pi$ a random permutation matrix. Furthermore, the matrix multiplication with a Hadamard matrix can be computed in $\mathcal{O}(D \log{d})$ via the Fast Walsh-Hadamard Transform. Note that everything but $\theta_d$ is fixed; therefore, the optimization problem lies only in $d$-dimensions. Note that if we place a constraint of $M$ being a binary matrix, we recover the sparsification problem; therefore, we can view finding intrinsic dimensionality as a continuous relaxation of the sparsification problem. 

The standard method of measuring the intrinsic dimensionality of an objective as proposed by \cite{intrinsic_dimension} requires searching over various $d$, training using standard SGD over the subspace reparameterization $\theta^{D}$ and selecting the smallest $d$ which provides us with a satisfactory solution ($d_{90}$). \cite{intrinsic_dimension} defined the \textit{satisfactory solution} as being 90\% of the full training metric. For example, if we reach 85\% accuracy training a model with all of its parameters, the goal is to find the smallest $d$, which would reach $0.9 * 85\% = 76.5\%$ accuracy; we call this dimension $d_{90}$. Let us also note that by merely initializing $\theta^d = 0$ we recover the original parameterization $\theta^{D}_0$ which in the context of fine-tuning represents the original weights of the pre-trained model.

The way \cite{intrinsic_dimension} define a satisfactory solution reduces the dependence of the dataset's size on the calculation of intrinsic dimension. For a small dataset, we will generally have worse end metrics; therefore, we have a lower $d_{90}$ cut-off; inversely, a larger dataset will require a more non-trivial $d_{90}$ cut-off.
\paragraph{Structure Aware Intrinsic Dimension}
%We are interested in applying this method to the fine-tuning objective of large pre-trained models to better to understand the complex properties of these pre-trained models. 
Due to the large size of pre-trained language models (generally in the hundreds of millions of parameters), the only computationally reasonable subspace optimization method is one that utilizes the Fastfood transform. For example, if we are interested in subspace training with $d=1000$ for the RoBERTa-Large model using a dense matrix, we would require 1.42 terabytes of memory to store just the projection matrix.

Unfortunately, the method of finding the intrinsic dimension proposed by \cite{intrinsic_dimension} is unaware of the layer-wise structure of the function parameterized by $\theta$. Existing literature argues that in attention-based pre-trained models, individual layers specialize separately \citep{what_does_bert_look_at}; therefore, it is useful to incorporate a notion of structure when computing $d_{90}$. 
We define Structure-Aware Intrinsic Dimension (SAID) as the following
\begin{equation}
    \theta^{D}_i = \theta^{D}_{0, i} + \lambda_i P(\theta^{d-m})_i
\end{equation}
For $m$ layers, we trade $m$ parameters from our subspace parameter $\theta_d$ to allow for layer-wise scaling through jointly learned $\lambda$, thus $\theta_d$ becomes $\left[\theta_{d-m}, \lambda\right]$. This allows the SAID method to focus a larger capacity of $\theta^{d-m}$ towards specific layers what might carry more relevant information for the task at hand. Conversely, we will refer to the layer unaware method (Equation~\ref{eq:did}) as the Direct Intrinsic Dimension (DID) method.
% TODO: possible add a section of related works to discuss things like lottery ticket hypothesis.
\section{Intrinsic Dimensionality of Common NLP Tasks}
\label{sec:mes}

\subsection{Sentence Prediction}
We first empirically calculate the intrinsic dimension of various pre-trained models on a set of sentence prediction tasks from the GLUE Benchmark \citep{GLUE}. We focus on analyzing BERT \citep{BERT} and RoBERTa \citep{ROBERTA} at both the base and large model sizes.

We chose to experiment with MRPC \citep{mrpc} and QQP \citep{qqp} as reference examples of small and large tuning datasets. MRPC is a binary classification task for predicting semantic equivalency for two paraphrases with roughly 3700 training samples, while QQP is a binary classification task for predicting semantic equality of two questions, with roughly 363k samples. For every dataset and every model, we run 100 subspace trainings with $d$ ranging from 10 to 10000 on a log scale. For every training run, we do a small hyperparameter search across four learning rates. We initialize every $\theta_d$ to the zero vector to allow for our starting point to be the original pre-trained model. Our subspace optimization method also operates over the randomly initialized sentence classification head to ensure we have exactly $d$ parameters to optimize.

We use both the SAID and DID subspace optimization methods, which we implemented in the Huggingface Transformers library \citep{huggingface}. We present the results in Figure~\ref{fig:sp_DID}.

\begin{figure}
  \centering
  \subfloat{
    \includegraphics[width=1.0\textwidth]{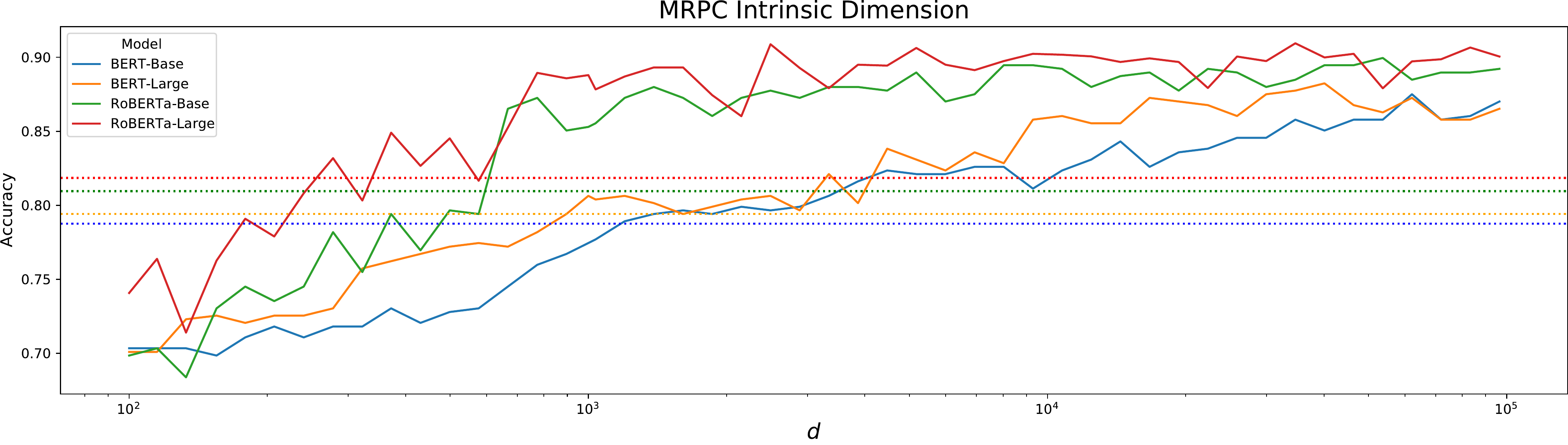}
    } \\
  \subfloat{
    \includegraphics[width=1.0\textwidth]{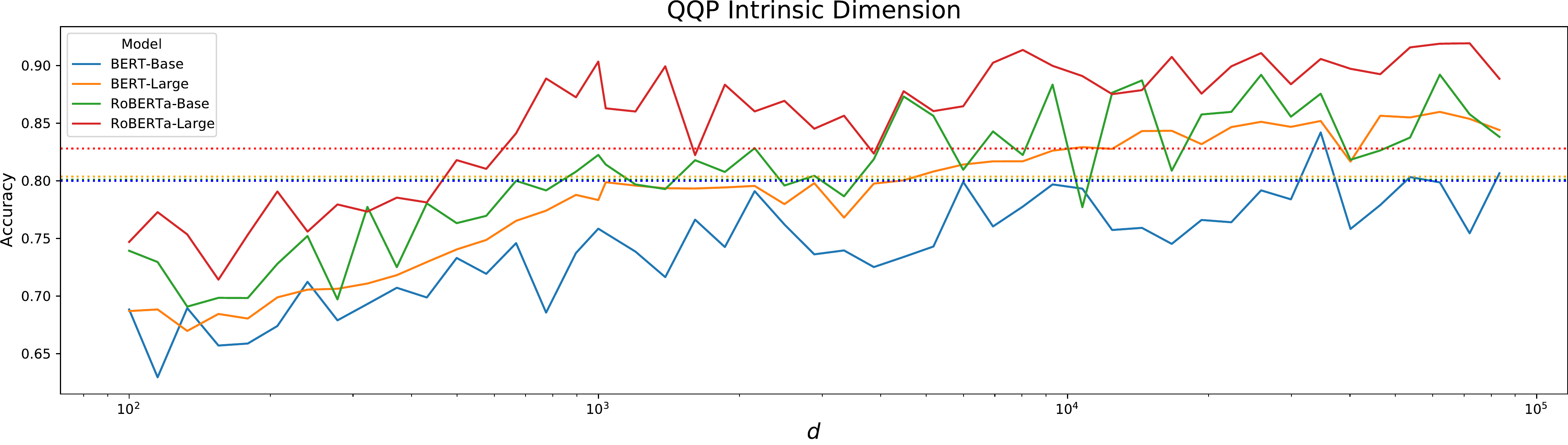}
    }
  \caption{The following figures show the evaluation accuracy on two datasets and four models across a range of dimensions $d$ for the DID method. The horizontal lines in each figure represent the 90\% solution of the respective full model.} \label{fig:sp_DID}
\end{figure}

\begin{wraptable}{r}{7.5cm}
\vspace{-2em}
\centering
\small
\begin{tabular}{lrrrr}\\\toprule 
    & \multicolumn{2}{c}{SAID}  & \multicolumn{2}{c}{DID} \\ \cmidrule(lr){2-3}\cmidrule(lr){4-5} 
Model & MRPC & QQP & MRPC & QQP \\\midrule
BERT-Base & 1608 & 8030 & 1861 & 9295 \\  
BERT-Large & 1037 & 1200& 2493 & 1389\\  \midrule
RoBERTa-Base & 896 & 896& 1000 & 1389 \\  
RoBERTa-Large & \bf{207} & \bf{774}& 322 & \bf{774}\\  \bottomrule
\end{tabular}
\caption{Estimated $d_{90}$ intrinsic dimension for a set of sentence prediction tasks and common pre-trained models. We present both the \textit{SAID} and \textit{DID} methods.}
\label{tab:mrpc_qqp_id}
\end{wraptable} 

\subsection{Analysis}
The first takeaway is the incredible low dimensionality of viable solutions. With RoBERTa-Large, we can reach 90\% of the full fine-tuning solution of MRPC using roughly 200 parameters and 800 parameters for QQP (Table~\ref{tab:mrpc_qqp_id}). Recall that our approximation of intrinsic dimension is necessarily crude by using random projections and restricting them to the use of Fastfood transform; therefore, it is likely that the true intrinsic dimension is much lower.

Furthermore, RoBERTa consistently outperforms BERT across various subspace dimensions $d$ while having more parameters. We leave a more in-depth analysis of model parameter size on intrinsic dimensionality to a later section (\S\ref{section:parameter_exploration}).

Lastly we see that adding a notion of structure in the computation of intrinsic dimension is beneficial with the SAID method consistently improving over the structure unaware DID method.

\section{Intrinsic Dimension, Pre-Training, and Generalization Gap}
One interpretation of the intrinsic parameter vector is that it encodes the task at hand with respect to the original pre-trained representations. Therefore, we can interpret $d$ as the minimal description length of the task within the framework dictated by the pre-trained representations \citep{min_desc_length}. Under this interpretation of intrinsic dimensionality, we hypothesize that pre-training is implicitly lowering the intrinsic dimensionality of the average NLP task, and therefore compress the minimal description length of those same tasks.

What do we more precisely mean by intrinsic parameter encoding a task within the framework provided by the pre-trained representations? Traditionally, a finetuned model (e.g. for a classification tasks) simply consists of a classification head $g$, parameterized by $w_g$ applied to fine-tuned representations $f$, parameterized by $w_f$ per sample $x$. Therefore, to fully describe a task, we need to pack together parameterizations and weights $\left\{g,f, w_g, w_f\right\}$. This model description is completely decoupled from the original weights of the pre-trained representation $w_{f_0}$, therefore to represent $n$ classification tasks, we need to maintain $n \left\{w_g, w_f\right\}$; additionally, the task representation is incredibly high dimensional. Conversely, fine-tuning utilizing SAID in $d$-dimensions requires storing only $\theta_d$ per task, a single random seed used to generate $M$ and the original pre-trained weights $w_{f_0}$. Therefore, we can represent arbitrary NLP tasks within a single pre-trained model framework with $d+1$ parameters.

For example, in the last section, we represented MRPC with roughly 200 parameters, which translates to needing less than a kilobyte of data to encode a complex natural language task within the framework provided by RoBERTa.

We hypothesize that the better the pre-trained models are, the fewer bits (description length) are needed to represent the average NLP task, as we will demonstrate empirically in the next section.

\subsection{Pre-Training Intrinsic Dimension Trajectory}
\label{sec:trajectory}
To verify our hypothesis of pre-training optimizing intrinsic dimension, we retrain a RoBERTa-Base from scratch and measure various NLP tasks' intrinsic dimensions using the SAID method across various checkpoints. We completely replicate the setting as described by \citep{ROBERTA} apart from only training for a total of 200k steps (instead of 500k) with half the batch size (1k). To calculate the intrinsic dimension more efficiently, we reuse the best learning rates discovered in Section~\ref{sec:mes} for $d < 10000$ and use a fixed learning rate for anything else. To find $d_{90}$ we do a binary search across $d$ per each checkpoint, with a minimum $d$ of 100 and a maximum of 4 million. The ``full solution" that we use when deciding $d_{90}$ cut-off is computed by fine-tuning the checkpointed model in the standard way. We compute SAID on six datasets; \textit{MRPC}, \textit{QQP}, \textit{Yelp Polarity} \citep{yelp_polarity}, \textit{SST-2} \citep{sst2}, \textit{MNLI} \citep{mnli} and \textit{ANLI} using all rounds of data \citep{anli}.

\begin{figure}
    \centering
    \includegraphics[width=1.0\textwidth]{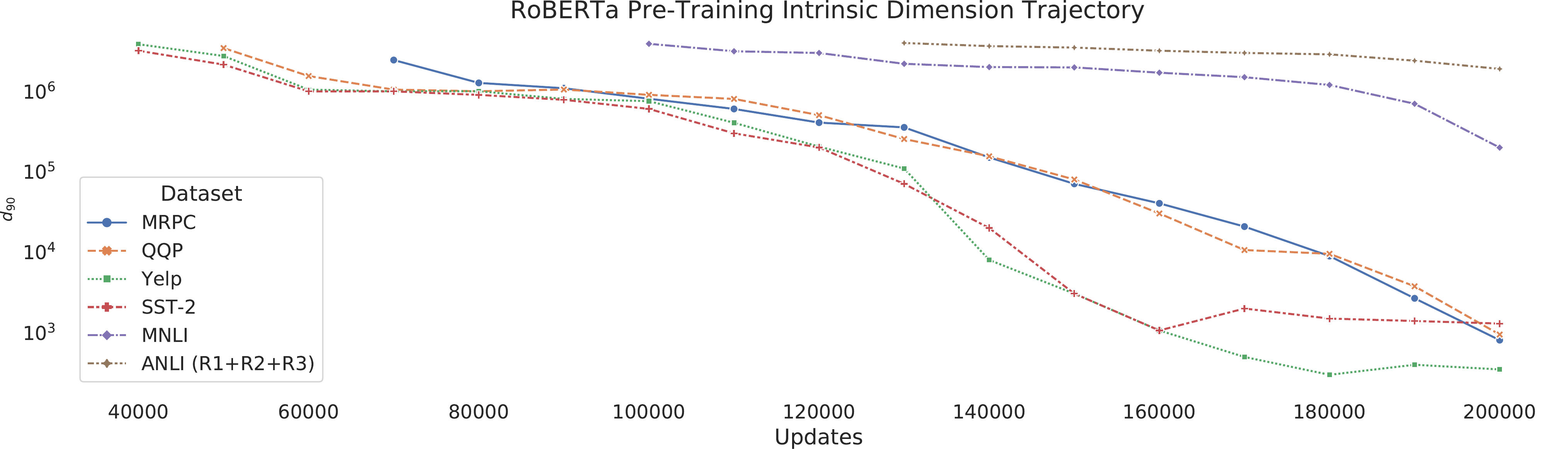}
    \caption{Every 10k updates of RoBERTa-Base that we trained from scratch, we compute $d_{90}$ for six datasets; MRPC, QQP, Yelp Polarity, SST-2, MNLI, and ANLI. If we were unable to compute a $d_{90}$ for a specific checkpoint, we do not plot the point, hence some datasets start at later points. Unable to compute means either we could not fine-tune the full checkpoint to accuracy above majority class or stabilize SAID training.}
    \label{fig:roberta_id}
\end{figure}

We present our results in Figure~\ref{fig:roberta_id}. We see that the intrinsic dimensionality of RoBERTa-Base monotonically decreases as we continue pre-training. We do not explicitly optimize for intrinsic dimensionality, specifically during pre-training (the language model does not have access to downstream datasets!), but none-the-less the intrinsic dimension of these downstream tasks continues to decrease.

More so, tasks that are easier to solve consistently show lower intrinsic dimensionality across all checkpoints, for example, \textit{Yelp Polarity} vs. the notoriously tough \textit{ANLI} dataset. The correlation between tasks traditionally hard for RoBERTa and their large intrinsic dimension hints at a connection between generalization and intrinsic dimension. We will discuss generalization further in Section~\S\ref{sec:generalization}.

Given our task representation interpretation of intrinsic dimensionality, we argue that the large scale training of Masked Language Models (MLM) learns generic and distributed enough representations of language to facilitate downstream learning of highly compressed task representations. Furthermore, we argue for another perspective of pre-training learning representations that form a compression framework with respect to various NLP tasks.

\subsection{Parameter Count and Intrinsic Dimension}
\label{section:parameter_exploration}
% TODO: Would despined figures be more aesthetic here? Maybe we should despine across all figures
We would also like to measure the relationships between the parameter count of arbitrary pre-trained models and the intrinsic dimension of downstream NLP tasks. The optimal experiment to run would be to fix the pre-training method, e.g., MLM RoBERTa style, vary the architecture size from small to very big, and compute the intrinsic dimension of a group of tasks at every size of the model. Unfortunately, such an experiment is computationally infeasible due to the need to train many RoBERTa models.

Due to these constraints, we opt to do an empirical study over existing pre-trained models, regardless of the pre-training method. We show that the trend is strong enough to overcome differences in training methodology. We select the following pre-trained models in our study: BERT \citep{BERT}, RoBERTa \citep{ROBERTA}, BART \citep{BART}, Electra \citep{ELECTRA}, Albert \citep{ALBERT}, XLNet \citep{XLNET}, T5 \citep{T5}, and XLM-R \citep{XLMR}. Furthermore, we selected various sizes of these models, as available publicly within the HuggingFace Transformers library~\citep{huggingface}.

We used the MRPC dataset and computed intrinsic dimension for every pre-trained model utilizing the same binary search methodology mentioned in the previous section with additional small hyper-parameter searches across learning rate (due to the wide range of learning rates needed by various models). 

\begin{figure}[h]
    \centering
    \includegraphics[width=1.0\textwidth]{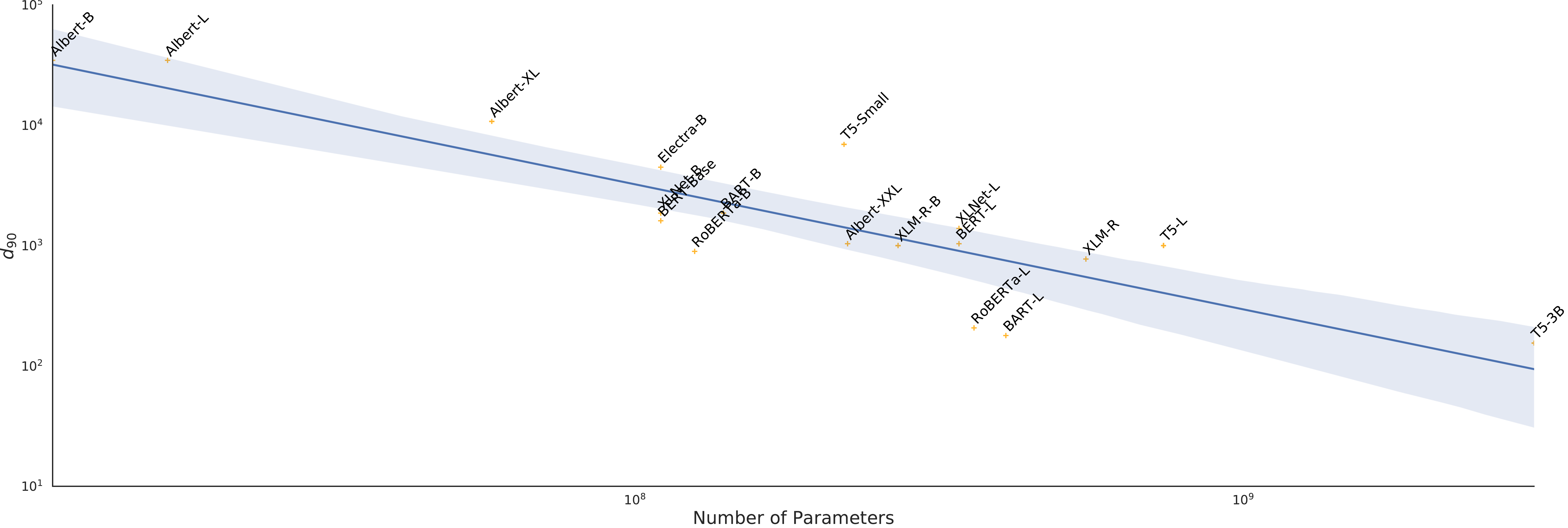}
    \caption{We calculate the intrinsic dimension for a large set of pre-trained models using the SAID method on the MRPC dataset.}
    \label{fig:mrpc_parameter_id}
\end{figure}

We present our results in Figure~\ref{fig:mrpc_parameter_id}. We see a strong general trend that as the number of parameters increases, the intrinsic dimension of fine-tuning on MRPC decreases. We ran this experiment on other datasets to ensure that this is not an artifact of the dataset. Our experiments showed the same trend; we refer to the Appendix for all trends per dataset.

Within the same window of number of parameters, pre-training methodology becomes essential. For example, in the regime of $10^8$ parameters, the RoBERTa method of pre-training dominates similar sized pre-training methods. However, there does not seem to be a method that can overcome the limitations induced by the number of parameters. Interpreting these results through the lens of learning a compression framework for NLP tasks is straightforward; the more parameters we have in the model, the less we need to represent a task.

\subsection{Generalization Bounds through Intrinsic Dimension}
\label{sec:generalization}
We have shown strong empirical evidence connecting pre-training, fine-tuning, and intrinsic dimensionality. However, we have yet to argue the connection between intrinsic dimensionality and generalization. Given that we have seen pre-training minimize intrinsic dimension, we hypothesize that generalization improves as the intrinsic dimension decreases.

To do so, we will empirically experiment with the connections between $d_{90}$ and evaluation set performance by looking at various checkpoints from our RoBERTa experiments in Section~\S\ref{sec:trajectory}. We also plot the relative generalization gap (delta between train time performance and test time performance).

\begin{figure}
    \centering
    \includegraphics[width=1.0\textwidth]{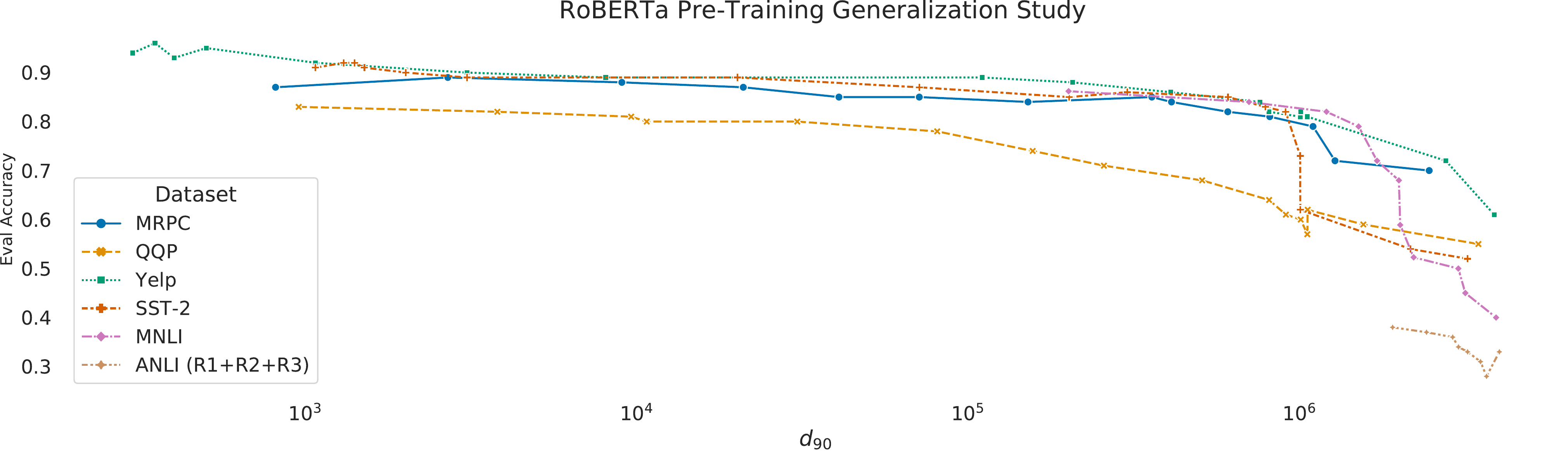}
    \caption{We plot the evaluation accuracy of six datasets across various intrinsic dimensionalities. There is a strong general trend that pre-trained models that are able to attain lower intrinsic dimensions generalize better.}
    \label{fig:roberta_gen_eval_acc}
\end{figure}

In Figure~\ref{fig:roberta_gen_eval_acc} we plot the evaluation accuracy's achieved by our pre-training experiment in Section~\S\ref{sec:trajectory}. A lower intrinsic dimension is strongly correlated with better evaluation performance. Additionally we are interested in measuring relative generalization gap ($\frac{acc_{train}-acc_{eval}}{1-acc_{eval}}$) across intrinsic dimension. We select the training accuracy that provides us with the best evaluation metrics when computing this figure. 
\begin{figure}
    \centering
    \includegraphics[width=1.0\textwidth]{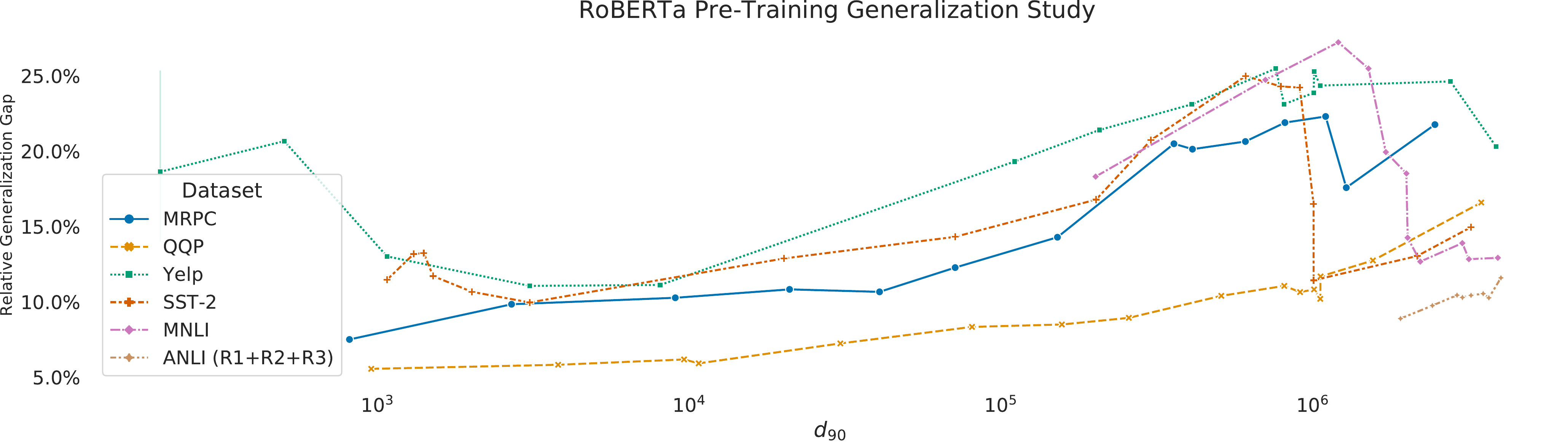}
    \caption{We plot the intrinsic dimension and the respective relative generalization gap across a set of varied tasks.}
    \label{fig:roberta_gen_rel_gap}
\end{figure}

We present our results in Figure~\ref{fig:roberta_gen_rel_gap}. Lower intrinsic dimension once again correlates strongly with a smaller relative generalization gap. If we interpret the intrinsic dimension as a measure of complexity, we expect the generalization gap to decrease with intrinsic dimension.

\subsubsection{Generalization Bounds}
By applying standard compression based generalization bounds, we can provide theoretical backing to the empirical connection between intrinsic dimension and generalization \citep{compression_generalization_gap}.

Consider the following definition of multi-class classification loss with an optional margin over our supervised dataset $D$. 
\begin{equation}
    \mathcal{L}_{\gamma}(f) = \mathbb{P}_{(x,y)\sim D}\left[f(x)[y] \le \gamma + \max_{i \ne y} f(x)[j] \right]
\end{equation}

When $\gamma = 0$, $\mathcal{L}_0$ recovers the standard classification loss. Furthermore, Let $\hat{\mathcal{L}}_{\gamma}(f)$ be an unbiased empirical estimate of the margin loss. 
\begin{theorem}
Let $f$ be a function which is parameterized by $\theta^D$ as described in Equation~\ref{eq:subspace_def} with a total of $d$ trainable intrinsic parameters on a dataset with $m$ samples. Then with a high probability, we can state the following asymptotic generalization bound
\begin{equation}
    \mathcal{L}_{0}(f) \leq \hat{\mathcal{L}}_{0}(f) + \mathcal{O}\left(\sqrt{\frac{d}{m}}\right)
\end{equation}
\end{theorem}
\begin{proof}
We defer the proof Section~\S\ref{sec:proofs} in the Appendix. We note that this is an extension of the well-known compression based generalization bound explored by \cite{compression_generalization_gap}.  
\end{proof}

This generalization bound is independent of the underlying parameter count ($D$) of the pre-trained model but depends on the ability to compress the downstream task ($d$). Moreover, given that our previous section shows larger models compress better, our bounds are aligned with general intuition and recent empirical evidence that larger pre-trained models generalize better. Explicitly, these bounds only apply to pre-trained methods trained with the intrinsic dimension subspace method; research has yet to show that standard SGD optimizes in this low dimensional space (although experimentally, this seems to be confirmed). We leave the theoretical contribution of showing SGD optimizes in this space, resembling something such as intrinsic subspace, for future work.

We want to highlight that generalization is not necessarily measured by the pre-trained model's parameter count or measure of complexity, but the pre-trained model's ability to facilitate the compression of downstream tasks. In some sense, if we want to compress downstream tasks better, we must expect pre-trained representations to have a considerable measure of complexity.

\section{Conclusion}
In conclusion, we proposed viewing the various phenomena surrounding fine-tuning and pre-training through the lens of intrinsic dimensionality. We empirically showed that common natural language tasks could be learned with very few parameters, sometimes in the order of hundreds, when utilizing pre-trained representations.  We provided an interpretation of pre-training as providing a compression framework for minimizing the average description length of natural language tasks and showed that pre-training implicitly minimizes this average description length.

We continued by doing an empirical study of existing pre-training methods and their respective intrinsic dimension, uncovering the phenomena that intrinsic dimensionality decreases as we increase the number of pre-trained representation parameters. This phenomenon provides some intuitions to the trend of growing pre-trained representations. We connected intrinsic dimensionality with generalization by first showing that pre-trained models with lower intrinsic dimensions across various tasks achieve higher evaluation accuracies and lower relative generalization gaps. Furthermore, we explain these empirical results by applying well-known generalization bounds to the intrinsic dimension to get generalization bounds that grow on the order of the intrinsic dimension, not on the pre-trained model's parameter count.

Intrinsic dimensionality is a useful tool for understanding the complex behavior of large models. We hope that future work will make explicit theoretical connections between SGD and optimizing the intrinsic dimension as well as explain exactly why pre-training methods optimize the intrinsic dimensionailty of tasks before not seen.

\bibliography{iclr2020_conference}
\bibliographystyle{iclr2020_conference}

\appendix
\section{Appendix}
\subsection{Proofs}
\label{sec:proofs}
\cite{compression_generalization_gap} define $(\gamma, S)$ compressible using helper string $s$ as the following.
\begin{definition}
$(\gamma, S)$ compressible using helper string $s$ 

Suppose $G_{\mathcal{A},s}= \left\{g_{\theta,s}|\theta \in \mathcal{A}\right\}$ is a class of
classifiers indexed by trainable parameters A and fixed strings s. A classifier $f$ is $(\gamma, S)$-compressible
with respect to $G_{\mathcal{A}}$ using helper string s if there exists $\theta \in \mathcal{A}$ such that for any $x \in S$, we have
for all y
\begin{equation}
|f(x)[y] - g_{\theta,s}(x)[y]| \leq \gamma 
\end{equation}
\end{definition}
\begin{remark}
If we parameterize $f(x; \theta)$ via the intrinsic dimension approach as defined in Equation~\ref{eq:subspace_def}, then $f$ is compressible losslessly using a helper string consisting of the random seed used to generate the static random projection weights and the initial pre-trained representation $\theta^D_0$. Therefore we say $f$ parameterized by either DID or SAID is $\left(0, S\right)$ compressible.
\end{remark}
Theorem~$2.1$ in \cite{compression_generalization_gap} states given a compression consisting of $r$ discrete states we achieve the following generalization bound.
\begin{equation}
    \mathcal{L}_{0}(f) \leq \hat{\mathcal{L}}_{\gamma}(f) + O\left(\sqrt{\frac{d \log{r}}{m}}\right)
\end{equation}

We can trivially represent our parameters $\theta_d$ in a discrete fashion through discretization (as was done in \cite{compression_generalization_gap}), and the number of states is dependent on the level of quantization but is static once chosen (FP32 vs. FP16).

We then connect the fact that models trained in low dimensional subspace using SAID/DID methods are (0, S)-compressible to derive the final asymptotic bound.

\begin{equation}
    \mathcal{L}_{0}(f) \leq \hat{\mathcal{L}}_{0}(f) + \mathcal{O}\left(\sqrt{\frac{d}{m}}\right)
\end{equation}
\end{document}